%% file: aistats_2018.tex
\documentclass[twoside]{article}

\usepackage[accepted]{aistats2019}
\usepackage[round]{natbib}
\renewcommand{\bibname}{References}
\renewcommand{\bibsection}{\subsubsection*{\bibname}}

\usepackage{graphicx} %
\usepackage{subfigure} 

\usepackage{amsmath}
\usepackage{mathrsfs}
\usepackage{amsthm}
\usepackage{stmaryrd}
\usepackage{enumitem}

\usepackage{thmtools, thm-restate}

\usepackage{algorithm}

\usepackage{hyperref}

\usepackage{url}            %
\usepackage{booktabs}       %
\usepackage{amsfonts}       %
\usepackage{nicefrac}       %
\usepackage{microtype}      %
\usepackage[noend]{algpseudocode}
\usepackage{caption}
\graphicspath{{images/}}

\include{notation}

\def \ld {l}
\def \p {\mathbf{p}}
\def \q {\mathbf{q}}
\def \z {\mathbf{z}}
\def \P {\mathbf{P}}
\def \Q {\mathbf{Q}}
\def \bP {\mathbb{P}}

\def \cTpi {{\cT^{\pi}}}
\def \Ppi {\Pr\nolimits_\pi}  %

\def \R {\mathbb{R}}
\def \PPi {\boldsymbol{\Pi}}
\def \lPPi {\boldsymbol{\Pi}_{\xi, \lambda, \Phi}}
\def \nPPi {{\hat {\boldsymbol{\Pi}}_{\xi, \lambda, \Phi}}}
\def \cX {\mathcal{X}}
\def \cD {\mathcal{D}}
\def \cP {\mathcal{P}}
\def \bd {\bar d}
\def \PiC {\Pi_C}  %

\def \Aat {A A^\top}  %
\def \eet {e e^\top}  %

\def \ldt {\ld_{\lambda}^2}
\def \ldl {\ld_{\lambda}^2}
\def \hldl {{\hat l}_{\lambda}^2}
\def \lxi {\ld_{\xi, \lambda}^2}
\def \llxi {\ld_{\xi, {\lambda}}^2}
\def \hllxi {{\hat \ld}_{\xi, {\lambda}}^2}
\def \ll {\boldsymbol{\ld}}

\def \Deq {\overset{D}{=}}

\def \cramer {Cram\'er }
\graphicspath{{images/}}

\begin{document}

\twocolumn[

\aistatstitle{Distributional reinforcement learning with linear function approximation}

\aistatsauthor{ Marc G. Bellemare \And Nicolas Le Roux \And Pablo Samuel Castro \And Subhodeep Moitra}
\aistatsaddress{Google Brain}

]

\begin{abstract}
Despite many algorithmic advances, our theoretical understanding of practical distributional reinforcement learning methods remains limited.
One exception is Rowland et al. (2018)'s analysis of the C51 algorithm in terms of the Cram\'er distance, but their results only apply to the tabular setting and ignore C51's use of a softmax to produce normalized distributions.
In this paper we adapt the Cram\'er distance to deal with arbitrary vectors. 
From it we derive a new distributional algorithm which is fully Cram\'er-based and can be combined to linear function approximation, with formal guarantees in the context of policy evaluation.
In allowing the model's prediction to be any real vector, we lose the probabilistic interpretation behind the method, but otherwise maintain the appealing properties of distributional approaches.
To the best of our knowledge, ours is the first proof of convergence of a distributional algorithm combined with function approximation. Perhaps surprisingly, our results provide evidence that Cram\'er-based distributional methods may perform worse than directly approximating the value function.
\end{abstract}

\section{Introduction}
In reinforcement learning one often seeks to predict the expected sum of discounted rewards, also called return or \emph{value}, of a given state. The distributional perspective on reinforcement learning takes this idea further by suggesting that we should predict the full distribution of this random return, called \emph{value distribution} \citep{bellemare2017distributional}. This has produced state-of-the-art performance on a number of deep reinforcement learning benchmarks \citep[e.g.][]{hessel18rainbow,barthmaron18distributional,dabney18implicit,dabney18distributional}.

The original distributional algorithm from this line of work is Bellemare et al.'s C51 algorithm.
Core to C51 are 1) the use of a softmax transfer function to represent the value distribution, 2) a heuristic projection step, and finally 3) the minimization of a Kullback-Leibler (KL) loss. \citet{rowland2018} showed that the heuristic projection minimizes a probability metric called the Cram\'er distance. However, their work did not explain the role of the KL loss in the algorithm.

The combination of two losses (Cram\'er and KL) is less than ideal, and makes the learning process technically more challenging to implement than, e.g., the classic Q-Learning algorithm \citep{watkins89learning}. This combination also makes it difficult to provide theoretical guarantees, both in terms of convergence but also in the quality of the value distribution generated by an approximate learner. %

A natural question is whether it is possible to do away with the softmax and KL loss, and derive a ``100\% Cram\'er'' algorithm, both for simplicity and theoretical understanding. In this paper we seek an algorithm which directly minimizes the Cram\'er distance between the output of the model, for example a deep network, and a target distribution. 
As it turns out, we can construct such an algorithm by treating the model outputs as an improper probability distribution, and deriving a variant of the Cram\'er distance which gracefully handles such distributions.

This new algorithm enables us to derive theoretical guarantees on the behaviour of a distributional algorithm when combined to linear function approximation, in the policy evaluation setting. Although convergence is guaranteed under the usual conditions, our performance bound is worse than that of an algorithm which only approximates the value function. This suggests that predicting the full distribution as an intermediate step in estimating the expected value could hurt performance. As a whole, our results suggest that the good performance of C51 cannot solely be attributed to a better-behaved loss function.

\section{Background}
\label{sec:background}

We consider an agent acting in an environment described by a finite Markov Decision Process $\left \langle \cX, \cA, \Pr, R, \gamma \right \rangle$ \citep{puterman94markov}. In this paper we study the policy evaluation setting, in which we assume a fixed policy $\pi$ mapping states to distributions over actions and consider the resulting state to state transition function $\Pr_\pi$:
\begin{equation*}
\Ppi(x' \cbar x) := \sum_{a \in \cA} \pi(a \cbar x) \Pr(x' \cbar x, a) .
\end{equation*}
We view the reward function $R$ as a collection of random variables describing the bounded, random reward received when an agent exits a state $x \in \cX$. The value distribution \citep{bellemare2017distributional} describes the random \emph{return}, or sum of discounted rewards, received when beginning in state $x$:
\begin{equation*}
Z^\pi(x) := \sum_{t=0}^\infty \gamma^t R(X_t)\quad X_0 = x, X_{t+1} \sim \Ppi(\cdot \cbar X_{t}) .
\end{equation*}
The expectation of the value distribution corresponds to the familiar \emph{value function} $V^\pi(x)$ \citep{sutton98reinforcement}. Similar to the value function satisfying the Bellman equation, $Z^\pi$ satisfies the distributional Bellman equation with an equality in distribution:
\begin{equation*}
Z^\pi(x) \Deq R(x) + \gamma Z^\pi(X') \qquad X' \sim \Ppi(\cdot \cbar x),
\end{equation*}
The distributional Bellman operator $\cTpi$ over value distributions is defined as
\begin{equation}
\cTpi Z(x) \overset{D}{:=} R(x) + \gamma \Ppi Z(x),\label{eqn:distributional_bellman_operator}
\end{equation}
where with some abuse of notation we write $\Ppi Z(x) := Z(X'), X' \sim \Ppi(\cdot \cbar x)$. The operator $\cTpi$ is a contraction mapping in the following sense:
let $d$ be a metric between probability distributions on $\bR$, and for two random variables $U, V$ denote by $d(U, V)$ the application of $d$ to their distributions. We define the maximal metric $\bd$ between two value distributions $Z_1$, $Z_2$ as
\begin{equation*}
\bd(Z_1, Z_2) := \sup_{x \in \cX} d(Z_1(x), Z_2(x)) .
\end{equation*}
Now, we say that $d$ is 1) \emph{sum invariant} if $d(A + U, A + V) \le d(U, V)$ for any random variable $A$ independent of $U$ and $V$, and 2) \emph{scale sensitive} of order $\beta$ if for all $c \in \bR$, $d(cU, cV) \le c^\beta d(U, V)$ \citep{bellemare2017cramer}. For any metric $d$ which satisfies both of these conditions (with $\beta > 0$), then $\cTpi$ is a contraction mapping with modulus $\gamma^\beta$ in the maximal metric $\bd$:
\begin{equation*}
\bd(\cTpi Z_1, \cTpi Z_2) \le \gamma^\beta \bd(Z_1, Z_2) .
\end{equation*}
Under mild assumptions and as a consequence of Banach's fixed point theorem, the process $Z_{k+1} := \cTpi Z_k$ converges to $Z^\pi$ in $\bd$.

\subsection{Metrics Over Distributions}

Let $\p$ and $\q$ be two probability distributions. The Kullback-Leibler (KL) divergence of $\q$ from $\p$ is
\begin{equation*}
D_{KL}(\p, \q) = \int_{-\infty}^\infty \p(t) \log \frac{\p(t)}{\q(t)} \; dt.
\end{equation*}
Note that the KL divergence is not properly a metric, but does define a loss function.
However, the KL divergence is not scale sensitive. Furthermore, it is infinite whenever $\p$ is not absolutely continuous with respect to $\q$, which can be problematic when designing a distributional algorithm with finite support: applying the Bellman operator to a discrete random variable typically changes its support.

The KL divergence is generally used in conjunction with a softmax transfer function which guarantees that $\q$ has unit mass; without this constraint, the minimizer of $D_{KL}$ may not be $\q = \p$. Furthermore, the KL divergence corresponds to the matching loss for the softmax function, guaranteeing that the resulting optimization is convex \citep[with respect to the softmax weights;][]{auer95exponentially}.

Unlike the KL divergence, the Cram\'er distance \citep{szekely02estatistics} is a proper distance between probability distributions. Given two distributions $\p$ and $\q$ over $\mathbb{R}$ with cumulative distribution functions $F_{\p}$ and $F_{\q}$, the Cram\'er distance is defined as
\begin{align}\label{eqn:cramer_distance}
D_C(\p, \q) &= \int_{-\infty}^{+\infty} (F_{\p}(t) - F_{\q}(t))^2 \; dt .
\end{align}
For the purposes of distributional reinforcement learning, the Cram\'er distance has a number of appealing properties. First, it is both sum invariant and scale sensitive of order $\beta = \tfrac{1}{2}$. Then, the \cramer distance can be minimized by stochastic gradient methods.

\subsection{Approximation in the Distributional Setting}
\label{sec:approx}

Let us write $\P^\pi(x)$ for the distribution of the random variable $Z^\pi(x)$. 
There are two common hurdles to learning $\P^\pi(x)$: first, we typically do not have access to a simulator, and must instead rely on sample transitions; second, we cannot in general store the value distribution exactly, and instead must maintain an approximation. 
These two issues have been well studied in the expected value setting of reinforcement learning \citep[see, e.g.][]{bertsekas96neurodynamic,tsitsiklis97analysis}, in particular relating the mean behaviour of sample-based algorithms such as TD \citep{sutton88learning} to their operator counterparts, including in the context of linear function approximation. This section provides analogous notation describing sample-based methods for distributional reinforcement learning.

With a tabular representation, where distributions are stored exactly, \citet{rowland2018} showed the existence of a \emph{mixture update} with step-size $\alpha$:
\begin{equation*}
\P(x) \gets \P(x) + \alpha (f_{r, \gamma} (\P(x')) - \P(x)) \; .
\end{equation*}
In this mixture update, $f_{r, \gamma} (\P(x'))$ is the distribution corresponding to the random variable $r + \gamma Z(x')$, $Z(x') \sim \P(x')$. 
This update rule converges to $\P^\pi$ under the usual stochastic optimization conditions.

\citet{rowland2018} also analyzed a mixture update for approximately tabular representations, when $\P$ is constrained to be a distribution over uniformly-spaced atoms (we will describe this parametrization in greater detail in the next section). The modified update incorporates a projection step $\PiC$ which finds the constrained distribution $\P(x)$ closest to $f_{r, \gamma} (\P(x'))$ in Cram\'er distance:
\begin{equation}\label{eqn:mixture_update_approximately_tabular}
\P(x) \gets \P(x) + \alpha (\PiC f_{r, \gamma} (\P(x')) - \P(x)) \; .
\end{equation}
This projection step is used in the C51 algorithm, which parametrizes $\P_\Theta$ using a neural network with weights $\Theta$ and whose final layer uses a softmax transfer function to generate the vector of probabilities $\P_\Theta(x)$. Ignoring second order optimization terms, the C51 update  is
\begin{equation}
\Theta \gets \Theta - \alpha \nabla_\Theta D_{KL} (\PiC f_{r, \gamma} (\P_{\tilde \Theta}(x')) \cdbar \P_\Theta(x)), \label{eqn:mixture_update_approximate}
\end{equation}
where the use of the KL divergence is justified as the matching loss to the softmax, and $\tilde \Theta$ is a ``target'' copy of $\Theta$ \citep{mnih15human}.

Although the update rule Eq.~\eqnref{mixture_update_approximate} works well in practice, it is difficult to justify. The KL divergence is not scale sensitive, and it is not clear that its combination with the Cram\'er projection and the softmax function leads to a convergent algorithm.%

To address this issue, here we consider an update rule which directly minimizes the Cram\'er distance:
\begin{equation}
\Theta \gets \Theta - \alpha \nabla_\Theta D_C (f_{r, \gamma} (\P_{\tilde \Theta}(x')), \P_\Theta(x)) .\label{eqn:mixture_update_approximate_cramer}
\end{equation}
By the matching-loss argument, this suggests doing away with the transfer function and measuring the loss with respect to linear outputs. At first glance this might seem nonsensical, as these may not form a valid probability distribution. Yet, as we will see, the Cram\'er distance can be extended to deal with arbitrary vectors.

\section{Generalizing the Cram\'er Distance}
\label{sec:unnormalized_cramer}

In this section we generalize the Cram\'er distance to vectors which do not necessarily describe probability distributions. We then transform this generalized distance to obtain a loss that is suited to the distributional setting. At a high level, our approach is as follows:
\begin{enumerate}[leftmargin=0.9cm,topsep=0.5pt,itemsep=-1ex,partopsep=2ex,parsep=2ex]
	\item We rewrite the \cramer distance between distributions with discrete support as a weighted squared distance between vectors;
	\item We show that this distance has undesirable properties when generalized beyond the space of probability distributions, and address this by modifying the eigenstructure of the weighting used in defining the distance;
	\item We further modify the distance into a loss which regularizes the sum of vectors towards 1. This modification is key in our construction of an algorithm that is theoretically well-behaved when combined with linear function approximation.
\end{enumerate}

We consider the space $\cD$ of distributions over returns with finite, common, bounded support $\z = \{z_1, z_2, \dots, z_k\}$ with $z_i \le z_{i+1}$. In this context, Eq.~\eqnref{cramer_distance} simplifies to a sum with simple structure:
\begin{equation*}
D_C(\p, \q) = \sum_{i=1}^{k-1} (F_\p(z_i) - F_\q(z_i))^2 (z_{i+1} - z_i)
\end{equation*}
with $\p$, $\q$ in $\cD$, where $F_\p$ is the cumulative distribution function of $\p$:
\begin{equation*}
F_\p(z_i) = \sum_{j=1}^i \p(z_j).
\end{equation*}
We shall also assume that $k$ is odd and $\z = \{\frac{1-k}{2}, \dots, \frac{k-1}{2}\}$, i.e. $z_i = \frac{2i - 1 - k}{2}$, $z_{i+1} - z_i = 1$. Without detracting from our results, this simplifies their exposition.

Let $\p := [ p_1, p_2, \dots p_k ]$ and $\q := [ q_1, q_2, \dots q_k ]$ denote the vectors associated with $z_1, z_2, \dots, z_k$, and write $C$ for the lower-triangular matrix of 1s:
\begin{align*}
C &= \left[
\begin{array}{ccccccc}
1 & 0 & 0 & \ldots & 0 & 0 & 0\\
1 & 1 & 0 & \ldots & 0 & 0 & 0\\
1 & 1 & 1 & \ldots & 0 & 0 & 0\\
\vdots &  &  & \ddots &  &  & \vdots\\
1 & 1 & 1 & \ldots & 1 & 0 & 0\\
1 & 1 & 1 & \ldots & 1 & 1 & 0\\
1 & 1 & 1 & \ldots & 1 & 1 & 1
\end{array}
\right] \; .
\end{align*}
If $\sum_i p_i = 1, p_i \ge 0$ (resp., $\sum_i q_i = 1, q_i \ge 0$), these can be viewed as the probabilities of a distribution over $\z$.
Then, $C \p$ is the cumulative distribution of $\p$, and the Cram\'er distance between $\p$ and $\q$ becomes
\begin{align}
\ld_{CC^\top}^2(\p, \q) 	&:= \norm{C \p - C \q}^2 = \norm{\p - \q}^2_{C C^\top} . \label{eq:cramer_1}
\end{align}
One can replace the cumulative distributions with the tail cumulative distributions to get
\begin{align}
\ld_{C^\top C}^2(\p, \q) 	&= \norm{C^\top \p - C^\top \q}^2 = \norm{\p - \q}^2_{C^\top C} \; . \label{eq:cramer_2}
\end{align}
If $\p$ or $\q$ do not correspond to proper probability distributions, the Cram\'er distance of Eq.~\ref{eqn:cramer_distance} may be infinite, while Eq.~\ref{eq:cramer_1} and \ref{eq:cramer_2} remain finite. This suggests the use of this definition when comparing vector-valued objects that are close to, or attempt to approximate distributions.

However, the two distances can disagree when $\p$ and $\q$ do not correspond to proper probability distributions.
Let $\sum_i p_i$ be the ``mass'' of $\p$, reflecting its relationship to the mass of a probability distribution.
If $\p$ and $\q$ have different mass, then $\ld_{CC^\top}^2(\p, \cdot) \ne \ld_{C^\top C}^2(\p, \cdot)$. 
The issue is that Eq.~\ref{eq:cramer_1} and \ref{eq:cramer_2} measure differently the difference in mass.

To resolve this discrepancy, we modify the Cram\'er distance to deal unambiguously with uneven masses.
This leads to a two-part distance: The first is insensitive to differences of total mass while the second only penalizes that difference. Let
\begin{align*}
e = [1/\sqrt{k},  \ldots, 1/\sqrt{k}]^\top \quad &\textrm{and} \quad\Pi_{e^\perp} = I_k - ee^\top \; ,
\end{align*}
our distance is
\begin{align}
\label{eq:centered}
\ld^2_\lambda(\p, \q) &:= (\p - \q)^\top \Pi_{e^\perp} CC^\top \Pi_{e^\perp} (\p - \q)\nonumber\\
&\qquad + \lambda \left((\p - \q)^\top e\right)^2 \; .
\end{align}
Denoting $C_\lambda = \Pi_{e^\perp} CC^\top \Pi_{e^\perp} + \lambda e e^\top$, we have
\begin{equation*}
\ld^2_\lambda(\p, \q) = (\p-\q)^\top C_\lambda(\p - \q) = \norm{\p - \q}^2_{C_\lambda} .
\end{equation*}
First, one may note that, when $\p$ and $\q$ have the same total mass, we have $\ld^2_{CC^\top} = \ld^2_{C^\top C} = \ld^2_\lambda(\p, \q)$ for all values of $\lambda$. As such, this new distance clarifies the behaviour for arbitrary vectors while being consistent with the existing \cramer loss for proper distributions. For any given distribution $\p$, the solution to
\begin{equation*}
\min_{\q \in \bR^k} \ld^2_\lambda(\p, \q)
\end{equation*}
is $\p$. On the other hand, if the minimization is done over a constrained set, $\lambda$ determines the magnitude of the penalty from the difference in total mass.

As we will later see, the distance $\ld_\lambda$, used as a loss, is not sufficient to guarantee good behaviour with linear function approximation. Instead, we define a related loss but with an explicit normalization penalty:
\begin{align}
\hldl(\p, \q) &= (\p - \q)^\top \Pi_{e^\perp} CC^\top \Pi_{e^\perp} (\p - \q)\nonumber\\
&\qquad + \lambda \left(\q^\top e - 1\right)^2 \; .\label{eq:unit_norm}
\end{align}
Intuitively, $\hat{\ld}_\lambda$ recognizes that a distribution-like object should benefit from having unit mass. In the context of the distributional Bellman operator, this is the difference between backing up the mass at successor states versus normalizing the state's distribution to sum to 1. However, $\hat{\ld}_\lambda$ does not define a distance proper, and our theoretical treatment of it in Section \ref{sec:convergence} will require additional care.

\section{Analysis}
\label{sec:cramer_properties}
We now explore, through a series of lemmas, properties of the \cramer distance of relevance to distributional reinforcement learning. Two of these results will be related to the minimization of the \cramer loss directly over distributions and two will be related to the use of linear function approximation.

\subsection{Optimization properties}
We begin by analyzing properties resulting from the minimization of $\ld^2_\lambda$ over $\q$, beginning with the approximately tabular setting (Section \ref{sec:approx}).

\subsubsection{Impact on optimization speed}
A well-known result in convex optimization states that, when minimizing a quadratic function $f$ with positive definite Hessian $H$ using a batch first-order method, e.g., Eq. \ref{eqn:mixture_update_approximate_cramer}, the convergence to the optimum is linear with a rate of $ 1 - \frac{1}{\kappa}$ where $\kappa$ is the condition number of $H$. Assuming we directly optimize the \cramer loss over $\q$ with such a method, the convergence rate would depend on the condition number of the matrix used, i.e. $CC^\top $ when using the \cramer loss $\ld_{CC^\top}^2$ or $C_\lambda$ when using the extended loss $\ldl$.

\begin{restatable}[Condition number]{lemma}{condnumber}
\label{lemma:cond_number}
Let $\mathcal{C}$ be the set of symmetric matrices $M$ for which $(\p-\q)^\top M (\p-\q) = (\p-\q)^\top CC^\top (\p-\q)$ for all proper distributions $\p$ and $\q$. Let $\kappa_{\min}(\mathcal{C})$ the lowest condition number attained by matrices $M$ in $\mathcal{C}$. Then all the matrices of the form $C_\lambda$ with $\lambda \in [\lambda_{k-1}(C_0), \lambda_1(C_0)]$, where $\lambda_{k-1}(C_0)$ and $\lambda_1(C_0)$ are the second smallest and largest eigenvalues of $C_0$, respectively, have condition number $\kappa_{\min}(\mathcal{C})$.
\end{restatable}
The proof of this result and the following may be found in the appendix.

Lemma~\ref{lemma:cond_number} shows that the optimal convergence rate is obtained for a potentially wide range of values for $\lambda$. As an example, for $k=51$, the condition number of $CC^\top$ is about 4296 while $\kappa_{\min}(\mathcal{C})$ is around 1053, about 4 times lower, and this is true for $\lambda$ in the range $[0.250, 263]$.

\subsubsection{Preservation of the expectation}
Although the prediction $\q$ may not be a distribution, it still makes sense to talk of the dot product between $\q$ and the support $\z$ as its ``expectation'': indeed, in many cases of interest the optimization procedure does yield valid distributions. In designing a full distributional agent, this generalized notion of expectation is also a natural way to convert $\q$ into a scalar value, e.g. for decision making.

This section discusses potential guarantees on the difference in expected return between $\p$ and $\q$ when $\q$ is the minimizer of the \cramer loss $\ld^2_\lambda$ over a restricted set. Typically, we will ask $\q$ to have a specific support but other constraints might include that $\q$ must be normalized or that some values of $\q$ cannot be modified.
Specifically, the following lemma studies the impact on the expectation when minimizing the \cramer loss over an affine subset.
\begin{restatable}[Expectation preserving]{lemma}{eprojection}
\label{lemma:e_projection}
Let $\p$ be an arbitrary distribution over a discrete support. Let $\Pi_{A, b}(\p)$ the projection of $\p$ onto the linear subset $\mathcal{S}_{A, b} = \left\{\q | A\q = b\right\}$. Then, if the first and the last columns of $A$ are equal, i.e. $A_1 = A_k$, then $\p$ and $\Pi_{A, b}(\p)$ have the same expectation.
\end{restatable}
Lemma~\ref{lemma:e_projection} covers projections onto specific supports, as used in C51, as well as constraints on the total mass of $\Pi_{A, b}(\p)$, for instance that the projection has unit mass. Here, we use $\z$ to denote both the support and the vector containing the elements of that support. More generally, the \cramer projection offers a certain amount of freedom on the constraints that can be enforced while still preserving the expectation. In particular, leaving the two boundaries unconstrained is enough to preserve the expectation.

\subsection{Linear function approximation}
We next quantify the behaviour of our generalized loss function when combined to linear function approximation. Section \ref{sec:convergence} is the main theoretical contribution of this paper: it shows that the combination of a loss based on Equation \ref{eq:unit_norm} together with linear approximation produces a stable dynamical system, and quantifies the approximation error that results from it.

\subsubsection{Two-step optimization}
In categorical distributional RL, the target distribution $\p$ is the product of an application of the distributional Bellman operator and does not usually have the same support as the parametrized output distribution $\q(\theta)$. Recall that $\cD$ is the set of distributions with support $\z$. C51 first projects $\p$ onto $\cD$, yielding
\begin{equation*}
\Pi_{\lambda, \cD}(\p) = \arg \min_{\textbf{u} \in \cD} \ldl(\p, \textbf{u}),
\end{equation*}
assuming $\p$ is a proper distribution. Then, as a second step in the update process, it minimizes the KL divergence between $\Pi_{\lambda, \cD}(\p)$ and $\q(\theta)$.

In our experiments we retain the projection onto $\z$ from the C51 algorithm, and subsequently minimize our loss with respect to this projection.
Doing so is equivalent to directly minimizing the \cramer loss, even when $\p$ is not a proper distribution. Extending the result from Lemma 3 of \citet{rowland2018}, we note that $\Pi_{\lambda, \cD}$ is an orthogonal projection for $\q(\theta) \in \cD$:
\begin{equation*}
\ldl(\p, \q(\theta)) = \ldl(\p, \Pi_{\lambda, \cD}(\p)) + \ldl(\Pi_{\lambda, \cD}(\p), \q(\theta)).
\end{equation*}
Taking the derivative of the two sides of this equation with respect to $\theta$, the parameters of the model, yields 
\begin{equation*}
\frac{\partial \ldl(\p, \q(\theta))}{\partial \theta} = \frac{\partial \ldl(\Pi_{\lambda, \cD}(\p), \q(\theta))}{\partial \theta}
\end{equation*}
and minimizing the distance with the projection of $\p$ onto the support $\z$ of $\q$ leads to the same gradients. With some additional care, the argument extends to the loss with a normalization penalty, $\hldl$.

\subsubsection{Convergence to a fixed point}\label{sec:convergence}
We are now ready to show the convergence of distributional RL in the context of linear function approximation. Recall that  a proof of convergence for C51 is hindered by the failure of the KL minimization process to be nonexpansive in the Cram\'er distance; as we will see, our result critically depends on the loss defined in Equation \ref{eq:unit_norm}.

We consider a feature matrix $\Phi \in \bR^{n \times m}$, with $n$ the number of states and $m$ the number of features, and a weight matrix $\Theta \in \bR^{m \times k}$. That is, we consider outputs of the form $\Q = \Phi \Theta \in \bR^{n \times k}$, which with some abuse of terminology we call value distributions. As before, we write $\Q(x)$ to denote the $k$-dimensional output for state $x \in \cX$.

We study a stochastic update rule of the form given by Eq.~\eqnref{mixture_update_approximate_cramer}, but where $D_C$ is replaced by the loss $\hldl$. When the states to be updated are sampled according to a distribution $\xi$, the expected behaviour of this update rule corresponds to an operator akin to a projection \citep{tsitsiklis97analysis}. In our setting, the operator minimizes the $\xi$-weighted Cram\'er loss derived from $\hldl$, denoted
\begin{equation*}
\hllxi(\P, \Q) := \sum_{x \in \cX} \xi(x) \hldl(\P(x), \Q(x)) .
\end{equation*}
We denote this operator by $\nPPi$ (the notation is made explicit in the appendix). Given a value distribution $\P \in \bR^{n \times k}$, the operator finds the value distribution in the span of $\Phi$ which minimizes $\hllxi(\P, \cdot)$:
\begin{equation*}
\nPPi \P = \Phi \Theta^* \quad \text{where} \quad \Theta^* = \argmin_{\Theta} \hllxi(\P, \Phi \Theta) .
\end{equation*}
Finally, our analysis is performed with respect the distance $\ldl$, rather than the loss $\hldl$ (which is not a distance). This leads to the $\xi$-weighted distance
\begin{equation*}
\llxi (\P, \Q) := \sum_{x \in \cX} \xi(x) \ldl(\P(x), \Q(x)),
\end{equation*}
with corresponding projection operator $\PPi_{\xi, \lambda, \Phi}$.

\begin{figure*}[htb]
\begin{center}
\includegraphics[width=1\textwidth]{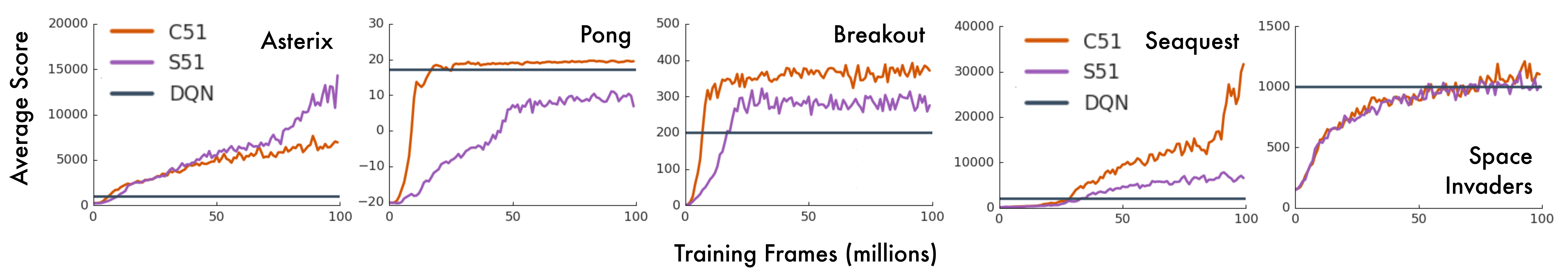}%
\caption{Learning curves (training scores) for C51 and S51 on five Atari 2600 games, and reference score for DQN at 100 million frames as given by \citet{bellemare2017distributional}.
\label{fig:c51_s51_results}}
\end{center}
\end{figure*}

We now show that the combination of the distributional Bellman operator $\mathcal{T}^\pi$ and the $\xi$-weighted, projection-like operator describes a convergent algorithm. When $\lambda > 0$, we can further bound the distance of this fixed point to the true value distribution $\P^\pi$ in terms of the best approximation in the class, $\PPi_{\xi, \lambda, \Phi} \P^\pi$. 
As is usual, $\xi$ is taken to be the stationary distribution of the Markov chain described by $\Ppi$: $\xi(x') = \sum_{x \in \cX} \xi(x) \Ppi(x' \cbar x)$.
\begin{restatable}[Convergence of the projected distributional Bellman process]{theorem}{convergencedistrib}
\label{theorem:convergencedistrib}
Let $\xi$ be the stationary distribution induced by the policy $\pi$. The process
\begin{align*}
\P_0 &:= \Phi \Theta_0 \quad , \quad \P_{k+1} := \nPPi \cTpi \P_k .
\end{align*}
converges to a set $S$ such that for any two $\P, \P' \in S$, there is a $\cX$-indexed vector of constants $\alpha$ such that
\begin{equation*}
\P(x) = \P'(x) + \alpha (x) e .
\end{equation*}
If $\lambda > 0$, $S$ consists of a single point $\tilde \P$ which is the fixed point of the process. Furthermore, we can bound the error of this fixed point with respect to the true value distribution $\P^\pi$:
\begin{align*}
\llxi(\tilde \P, \P^\pi) &\le \frac{1}{1-\gamma} \llxi(\lPPi \P^\pi, \P^\pi)\\
&\qquad - \frac{\gamma \lambda}{1 - \gamma} \big \| \tilde \P - \P^\pi \big \|^2_{\xi, \eet},
\end{align*}
where the second term measures the difference in mass between $\tilde \P$ and $\P^\pi$.
\end{restatable}
Theorem~\ref{theorem:convergencedistrib} is significant for a number of reasons. First, it answers the question left open by~\citet{rowland2018}, namely whether a proof of convergence exists for the distributional setting with an approximate representation, and with which representation. Second, it shows that there is a trade-off between the different components of the loss -- while our result concerns linear function approximation, it suggests that similar trade-offs must exist within other distributional algorithms.

The parameter $\lambda$ plays an important role in the theorem, both to guarantee convergence and (indirectly) to determine the approximation error. At a high level, this makes sense: a high value of $\lambda$ forces the algorithm to output something close to a distribution, at the expense of actual predictions. On the other hand, taking $\lambda = 0$ yields a process which may not converge to a single point. Finally, we note that to guarantee convergence to a unique fixed point, it is not enough to use the loss from Eq.~\ref{eq:centered}: in that case, we can only guarantee convergence to the set $S$, even for $\lambda > 0$. The following lemma, used to prove Theorem \ref{theorem:convergencedistrib}, shows why: the distributional Bellman operator $\cT$ is only a nonexpansion along the dimension $e$, which captures the mass of the output vectors.

\begin{restatable}{lemma}{xiweighted}\label{lemma:xiweighted}
Let $\xi$ be the stationary distribution induced by the policy $\pi$. Write $\cTpi' := \Pi_{\lambda, \cD} \cTpi$ to mean the distributional Bellman operator followed by a projection onto the support $\z = z_1, \dots, z_k$. For a matrix $B \in \bR^{k \times k}$ and $\Delta \in \bR^{n \times k}$, write
\begin{equation*}
\norm{\Delta}^2_{\xi, B} = \sum_{x \in \cX} \xi(x) \norm{\Delta(x)}^2_B .
\end{equation*}
Then for any two value distributions $\P, \Q \in \bR^{n \times k}$,
\begin{align*}
\norm{\cTpi' \P - \cTpi' \Q}^2_{\xi, \Aat} \le \gamma \norm{\P - \Q}^2_{\xi, \Aat}\\ \norm{\cTpi' \P - \cTpi' \Q}^2_{\xi, \eet} \le \norm{\P - \Q}^2_{\xi, \eet} .
\end{align*}
where $A := \Pi_{e^\perp} C$.
\end{restatable}
When $\P$ and $\Q$ have equal mass, we recover the contraction result by \citet{bellemare2017cramer} (albeit in $\xi$-weighted Cram\'er distance, rather than maximal Cram\'er distance) -- however, this also shows that our generalization of the distributional Bellman operator deals differently with probability mass itself. This is why Theorem \ref{theorem:convergencedistrib} requires the normalization penalty loss $\hldl$, rather than the simpler $\ldl$.

\subsubsection{Bound on the approximation error}

Our analysis provides us with a partial answer to the question: why and when should distributional reinforcement learning perform better empirically? In the linear approximation case that we study here, one answer is that it might hurt performance, as the following theorem suggests:
\begin{restatable}[Error bound for the expected value]{theorem}{convergenceexpectation}
\label{theorem:convergenceexpectation}
Let $\norm{\cdot}_\xi$ be the $\xi$-weighted norm over value functions. The squared expectation error of the fixed point $\tilde \P$ with respect to the true value function $V^\pi$ is bounded as
\begin{equation*}
\norm{\expects_{\tilde \P} \z - V^\pi}^2_{\xi} \le \|C_\lambda^{-1/2}\z\|^2 \ll^2_{\xi, \lambda}(\tilde \P, \P^\pi) .
\end{equation*}
\end{restatable}
The proof relies on a Rayleigh quotient argument, and shows that the bound is tight if the error vector $\tilde \P(x) - \P^\pi(x)$ is collinear with $C_\lambda^{-1/2} \z$. In particular, if we take $\lambda$ such that $C_\lambda = C C^T$, then the constant is $\|C_\lambda^{-1/2} \z\|^2 = \|e \sqrt{k} \|^2 = k$. Then, as $\lambda \to 0$, the constant goes to infinity.
By contrast, the bound on the approximate value function derived from~\citet{tsitsiklis97analysis} is better in two respects: first, its equivalent constant is 1. Second, our bound contains an amplification factor $1 / \sqrt{1 - \gamma}$ from the error term $\ll^2_{\xi, \lambda}(\tilde \P, \P^\pi)$, which in their bound becomes the smaller constant $1 / \sqrt{1 - \gamma^2}$, because the usual Bellman operator is a $\gamma$-contraction in $\norm{\cdot}_\xi$, while the Cram\'er distance is only a $\sqrt{\gamma}$-contraction in the equivalent norm.

However, the bound is slightly misleading. In our analysis we have assumed that the width of the support, i.e. $z_k - z_1$, also grows with $k$. We can instead normalize the $C$ matrix and the support $\z$ to reflect a fixed width: $C' = C / k$ and $z' = z / k$. In this case, the constant remains but the squared loss may in some cases be $k$ times smaller.
Still, it is not unreasonable to expect that, given that the distributional approach models more things, it should be more susceptible to misspecification.

\section{Experiments}
\label{sec:experiments}

The \cramer distance enjoys many theoretical properties that the KL divergence used in C51 lacks. To complement our theoretical results in the policy evaluation setting, we now study how our new loss affects the overall performance in the more complex control setting \citep{sutton98reinforcement}. 
Our goals are to demonstrate that we can achieve qualitatively comparable performance to C51 with an algorithm based on this loss, and to study the similarities and differences between the two algorithms. 

We compare the original C51 algorithm with our \cramer variant from Eq.~\ref{eq:unit_norm}, dubbed S51, on five games supported by the Arcade Learning Environment~\citep{bellemare2013arcade}, and using the Dopamine framework \citep{castro18dopamine}. In a nutshell, S51 learns from samples, using the sample-based version of the distributional Bellman operator (Eq.~\ref{eqn:distributional_bellman_operator}), but where the fixed policy is replaced by one which backs up the distribution with maximum expected value (what \citet{rowland2018} calls ``Categorical Q-Learning'').
Further experimental details, including on how to transform C51 into S51, are given in Appendix \ref{sec:experimental_details}.
\begin{figure*}[htb]
\begin{center}
\includegraphics[width=.95\textwidth]{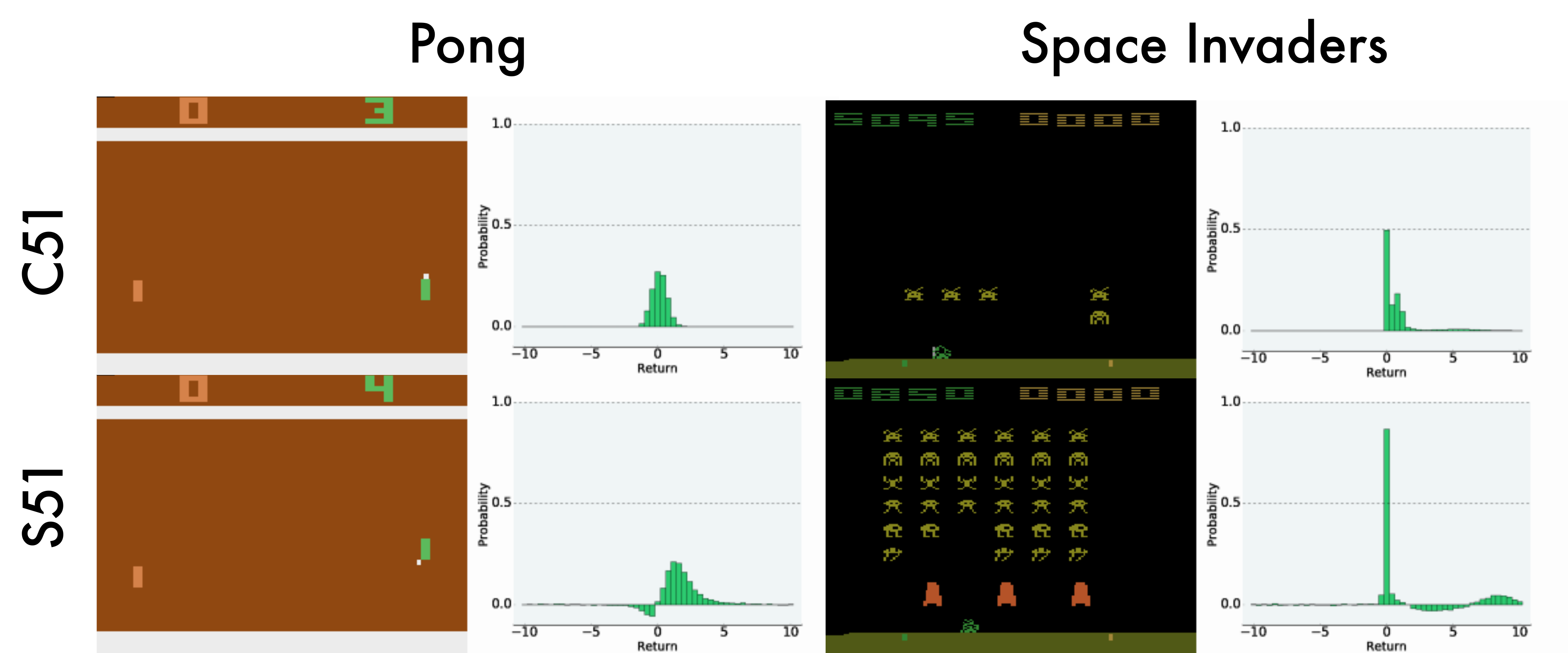}%
\caption{Distributions predicted by both algorithms in similar situations.
\label{fig:c51_s51_similar_situations}}
\end{center}
\end{figure*}

Figure \ref{fig:c51_s51_results} shows that S51 achieves higher scores than DQN, demonstrating that it maintains the empirical benefits of the distributional perspective, and performs as well as C51 in three out of five games. This is especially significant given the relative freedom of the network in outputting arbitrary vectors. Nonetheless, our results suggest that there are benefits to enforcing normalized distributions -- possibly in reducing the update variance.

To better understand the qualitative differences between the two algorithms, we studied agents playing through episodes of different games and visualized the predicted distribution for their selected actions (videos available in the supplemental; Figure \ref{fig:videos}). We find that C51 outputs value distributions which are bell-shaped and may have a separate mode at 0. In contrast, the S51 distributions are much more diverse; we highlight two interesting results:

\noindent \textbf {Double negatives}. S51 agents often assign negative mass to negative returns in games where such returns are impossible, such as \textsc{Pong} (Figure \ref{fig:c51_s51_similar_situations}, left). The total mass in these cases is still close to 1.

\noindent \textbf {Compensation around 0}. In \textsc{Space Invaders} (Figure \ref{fig:c51_s51_similar_situations}, right), the 0 return prediction is bracketed with small negative and positive corrections that cancel each other out.
One explanation is that the network compensates for its limited capacity by relying on negative return predictions. This is particularly interesting as this behaviour is not possible under published distributional algorithms.%

\noindent \textbf {Noisier predictions} (left and right). S51 assigns a small amount of probability to almost all returns. We hypothesize that this effect is visually absent from the C51 histograms because of the squashing effect of the softmax transfer function, and that this added noise explains some of the difference in performance. In particular, to generate a small probability the C51 network need only output a sufficiently negative logit; by contrast, S51 must output a value which is neither too negative nor too positive (i.e., is actually close to 0).

\section{Discussion and Conclusion}

While the convergence of the distributional approach with linear approximation may have been predictable, our proof shows that the result is not completely straightforward, and that the normalization penalty plays an important role in convergence. Because the softmax produces bounded outputs, it may still be possible to derive some convergence guarantees for it; however, it seems difficult to bound on its approximation error once we leave the convex regime of the linear outputs/squared loss combination.
Another question is whether minimizing the Cram\'er distance in the context of function approximation for optimal control somehow results in learning dynamics that are more stable than in the expected case, as a wealth of empirical results now suggest.

The Wasserstein distance also plays an important role in distributional reinforcement learning. \citet{dabney18distributional} demonstrated that one can obtain a stable distributional algorithm which minimizes the Wasserstein distance even in the approximate case by performing quantile regression rather than gradient descent on the sample Wasserstein loss. %
A similar analysis to ours may in fact prove convergence in the approximate setting; we expect that minimizing the Wasserstein metric should also be susceptible to pathological cases yielding a worse approximation of expected values.

Despite our attempts, we could not match the raw performance of C51. While this may only be a matter of hyperparameter tuning, we might have lost other properties when moving away from the KL. One might also wonder if there are other losses even more suited to the problem than our modified \cramer loss. In particular, since the ultimate goal is to preserve the expectation of the target distribution, one could adapt the loss to strengthen the link between loss minimization and expectation preservation. %

\bibliography{aistats_2018}
\bibliographystyle{icml2018}

\clearpage
\appendix
\section*{Appendix}

We now provide the proofs of our lemmas.

\subsection*{Impact on optimization speed}

To prove Lemma~\ref{lemma:cond_number}, we first need to prove an additional lemma.
\begin{restatable}[Set of \cramer extensions]{lemma}{cramerextensions}
\label{lemma:cramer_extensions}
The set $\mathcal{C}$ of symmetric matrices $M$ such that $(\p-\q)^\top M (\p-\q) = (\p-\q)^\top CC^\top (\p-\q)$ for all normalized distributions $\p$ and $\q$ is the set
\begin{align*}
\mathcal{C} &= \left\{ CC^\top + ae^\top + ea^\top | a \in \mathbb{R}^k\right\} \; .
\end{align*}
\end{restatable}
\begin{proof}
Let $M = CC^\top + ae^\top + ea^\top$. Since $\p$ and $\q$ are normalized, we have $(\p-\q)^\top e = 0$. Hence,
\begin{align*}
&(\p-\q)^\top M (\p-\q) = \\
&\qquad (\p-\q)^\top CC^\top(\p-\q)\\
&\qquad + (\p-\q)^\top ae^\top(\p-\q) + (\p-\q)^\top ea^\top(\p-\q)\\
&= (\p-\q)^\top CC^\top(\p-\q) \; .
\end{align*}
  Conversely, let a symmetric matrix $M$ be such that $(\p-\q)^\top M (\p-\q) = (\p-\q)^\top CC^\top(\p-\q)$ for all normalized $\p$ and $\q$. Then $(\p-\q)^\top (M - CC^\top) (\p-\q) = 0$. For this to be true, $(M - CC^\top) (\p-\q)$ must be colinear to $e$. Thus, denoting $M - CC^\top = ea^\top + N$ where $e$ is not in the span of $N$, we must have $N(\p-\q) = 0$ for all normalized $\p$ and $\q$, i.e. $N = be^\top$. The symmetry constraint leads to $a = b$. This concludes the proof.
\end{proof}

\condnumber*
\begin{proof}
  Let $v_L$ be the vector associated with the maximum eigenvalue $L$ of $C_0$ and $a$ be an arbitrary vector. Because $C_0$ is a symmetric matrix whose only zero eigenvalue corresponds to $e$, its eigenvectors are orthogonal to $e$ and in particular $v_L^\top e = 0$. Thus, we have
\begin{align*}
L	&= v_L^\top C_0 v_L\\
	&= v_L^\top \Pi_{e^\perp} CC^\top \Pi_{e^\perp} v_L\\
	&= v_L^\top CC^\top v_L\\
	&= v_L^\top CC^\top v_L + v_L^\top ae^\top v_L + v_L^\top ea^\top v_L
\end{align*}
for any vector $a$ since $e^\top v_L = 0$.
Denoting $R_a = CC^\top + ae^\top + ea^\top$, we get
\begin{align*}
L	&= v_L^\top R_a v_L\\
	& \leq \max_{v} \frac{v^\top R_a v}{\|v\|^2} \; ,
\end{align*}
which is the largest eigenvalue of $R_a$. Since this is true for every $a$, $C_0$ has the lowest top eigenvalue from all the matrices in $\mathcal{C}$.

Similarly, let us denote $v_\mu$ be the vector associated with the second-smallest eigenvalue $\mu$~\footnote{The smallest being 0.}. As $e$ is the eigenvector associated with the eigenvalue $0$, we have that $v_\mu ^\top e = 0$ and
\begin{align*}
\mu	&= v_\mu^\top C_0 v_\mu\\
	&= v_\mu^\top CC^\top v_\mu\\
	&= v_\mu^\top CC^\top v_\mu + v_\mu^\top ae^\top v_\mu + v_\mu^\top ea^\top v_\mu\\
	&= v_\mu^\top R_a v_\mu\\
	& \geq \min_{v} \frac{v^\top R_a v}{\|v\|^2} \; .
\end{align*}
  Thus, for all $a$, the second smallest eigenvalue of $C_0$ is larger than the smallest eigenvalue of $R_a$.

This means that, for $C_\lambda$ to have the smallest condition number of all the matrices in $\mathcal{C}$, it is sufficient to require that the eigenvalue associated with $e$ be between $\mu$ and $L$, i.e. that $\mu \leq \lambda \leq L$.
This concludes the proof.
\end{proof}

\subsection*{Preservation of the expectation}

To prove Lemma~\ref{lemma:e_projection}, we will need the following proposition:
\begin{restatable}[]{proposition}{cz}
\label{prop:cz}
Let $\z$ be defined as in Section~\ref{sec:unnormalized_cramer}, i.e. $\z$ is the vector of evenly spaced returns between $-\frac{k-1}{2}$ and $\frac{k-1}{2}$ with mean $0$. Let $b = [-1, 0, 0, \ldots, 0, 0, 1]^\top$. Then $C_\lambda^{-1}\z = b$ for all values of $\lambda > 0$.
\end{restatable}

\begin{proof}
We will prove that $C_\lambda b = \z$ for all values of $\lambda$. First, we note that $e^\top b = 0$ and $C_\lambda b = \Pi_{e^\perp} CC^\top b$.

Since $C_{ij} = 1_{i \geq j}$, we have, denoting $c = C^\top b$,
\begin{align*}
c_j &= \sum_i C_{ij} b_i\\
				&= \left\{\begin{array}{cl}0 & \textrm{if j = 1}\\ 1 & \textrm{otherwise}\end{array}\right. \; .
\end{align*}
Multiplying by $C$ to get $d = C c$, we get
\begin{align*}
d_i &= \sum_j C_{ij} (C^\top b)_j\\
				&= i-1\; .
\end{align*}

  We now need to compute $r = \Pi_{e^\perp} d$. Since $e^\top d = \frac{(k-1)\sqrt{k}}{2}$, we have
\begin{align*}
r_i &= d_i - V\\
	&= i - 1 - \frac{k-1}{2}\\
	&= \frac{2i-1 -k}{2}\\
	&= \z_i \; .
\end{align*}
This concludes the proof.
\end{proof}

\eprojection*
\begin{proof}
By definition, 
\begin{align*}
\Pi_{A, b}(\p) &= \begin{array}{cl}\arg \min_{\q} &(\p - \q)^\top C_\lambda (\p - \q)\\
					\textrm{subject to}& A\q = b \; .
					\end{array}
\end{align*}
Writing $\nu$ the Lagrange multipliers, this is a quadratic program whose solution is given by
\begin{align*}
\left[
\begin{array}{c}
\Pi_{A, b}(\p)\\
\nu
\end{array}
\right] &= 
\left[
\begin{array}{cc}
C_\lambda & A^\top\\
A & 0
\end{array}
\right]^{-1}
\left[
\begin{array}{c}
C_\lambda \p\\
b
\end{array}
\right] \; .
\end{align*}
Inverting the block diagonal matrix yields
\begin{align*}
\left[
\begin{array}{cc}
C_\lambda & A^\top\\
A & 0
\end{array}
\right]^{-1}
&= 
\left[
\begin{array}{cc}
M_{11} & M_{12}\\
M_{21} & M_{22}
\end{array}
\right]
\end{align*}
with
\begin{align*}
M_{11} &= C_\lambda^{-1} -  C_\lambda^{-1}A^\top (AC_\lambda^{-1}A^\top)^{-1}AC_\lambda^{-1}\\
M_{12} &= C_\lambda^{-1}A^\top (AC_\lambda^{-1}A^\top)^{-1}\\
M_{21} &= (AC_\lambda^{-1}A^\top)^{-1}AC_\lambda^{-1}\\
M_{21} &= - (AC_\lambda^{-1}A^\top)^{-1} \; .
\end{align*}
Hence,
\begin{align*}
\Pi_{A, b}(\p) &= M_{11} C_\lambda \p + M_{12} b\\
&= \p - C_\lambda^{-1}A^\top s
\end{align*}
for some $s$.
Thus, the expected $Q$-value with respect to the projected distribution is equal to
\begin{align*}
\z^\top \Pi_{A, b}(\p) &= \z^\top \p - \z^\top C_\lambda^{-1}A^\top s
\end{align*}
and the two expectations will be equal if $\z^\top C_\lambda^{-1}A^\top s = 0$.
Using Proposition~\ref{prop:cz}, we know that $C_\lambda^{-1}\z = b$. Thus, if it sufficient to have $Ab = 0$ for the two expectations to match. Since only the first and the last components of $b$ are nonzeros and they are opposite of each other, we have $Ab = 0 \Leftrightarrow A_1 = A_k$ when denoting $A_j$ the $j$-th column of $A$. This concludes the proof.
\end{proof}

\subsection*{Convergence to a fixed point}
This result requires additional definitions. A value distribution $\P$ maps states $x \in \cX$ to distributions on
$\bR$; we extend this to vectors defined by a linear combination of features:
\begin{align*}
\P(x) &= \Theta^\top \phi(x) \; ,
\end{align*}
where $\phi(x) \in \bR^m$ is the feature vector at state $x$ and $\Theta \in \R^{m \times k}$ is the parameter matrix we try to estimate.

Concatening all feature vectors into a feature matrix $\Phi \in \bR^{n \times m}$, our linear approximation is $\P_\Theta := \Phi \Theta \in \R^{n \times k}$. We assume that the vector $\P_\Theta(x) \in \R^k$ approximates a distribution over the support $\z := \{ z_1, z_2, \dots, z_k \}$, but it may have negative components and is not necessarily normalized.

We are given a distribution $\xi$ on $\cX$ and we shall use a \cramer distance between distributions over $\z$:
\begin{align*}
\ldl(\p, \q) := \norm{\p - \q}^2_{C_\lambda} \; .
\end{align*}
We transform the matrix $C_\lambda$ into an operator over continuous distributions, where with some abuse of notation we view $\p$ as a distribution over a finite set of Diracs: $\p(y) = \sum_{i} p_i \delta_{z_i = y}$. Then
{\small
\begin{align}
\Pi_{e^\perp}\p(x) &= \p(x) - \int_{y=z_1}^{z_k} \p(y) \; dy \nonumber\\
\Pi_{e^\perp}\q(x) &= \q(x) - \int_{y=z_1}^{z_k} \q(y) \; dy \nonumber\\
\ldl(\p, \q) &= \int_{x=z_1}^{z_k} \left(\int_{y=z_1}^x \left[\Pi_{e^\perp}\p(y) - \Pi_{e^\perp}\q(y)\right]\;dx\right)^2 dy\nonumber\\
&\qquad + \lambda \left(\int_{y=z_1}^{z_k} \left[\p(y) - \q(y)\right]\;dx\right)^2 \; . \label{eq:d_lambda_operator}
\end{align}
}
The first term on the right-hand side of Eq.~\eqref{eq:d_lambda_operator} penalizes the difference in cdf of $\p$ and $\q$ while the second term penalizes the difference in mass. When applied to two distributions $\p$ and $\q$ over $\z$, this is equivalent to $(\p - \q)^\top C_\lambda (\p - \q)$.
We define the weighted \cramer distance over value distributions by
\begin{equation*}
\llxi (\P, \Q) := \sum_{x \in \cX} \xi(x) \ldl(\P(x), \Q(x)) .
\end{equation*}

In what follows we identify three spaces of distributions or distribution-like objects. First, $\bP$ is the space of distributions with support the interval $[z_1, z_k)$. $\cD$ is the space of distributions over $\z$. Finally, $\cP$ is the vector space spanned by the features $\Phi \in \bR^{n \times m}$, that is: $\cP = \{ \Phi \Theta : \Theta \in \bR^{m \times k} \}$.

While our value distribution will only output distributions over the support $\z$, the distributional Bellman operator $\cTpi$ transforms distributions over $\z$ into distributions from $\bP$. We thus need to consider the projection $\Pi_{\lambda, \cP}$ which projects $\bP$ onto $\cD$:
\begin{align*}
\Pi_{\lambda, \cD}\p &= \arg \min_{\q \in \cD} \ldl(\p, \q) \; .
\end{align*}
Lemma 3 from~\citet{rowland2018} states that, for any distribution $\p \in \cD$, we have
\begin{align}
\label{eq:rowland}
\ldl(\p, \q) &= \ldl(\p, \Pi_{\lambda, \cD} \p) + \ldl(\Pi_{\lambda, \cD} \p, \q) .
\end{align}

We now move from the projection of distributions to the projection of value distributions. We define a projection in $\llxi$ of a value distribution $\Q$ onto the subspace $\mathcal{V}$ by
\begin{defn}
The $\xi$-weighted projection onto $\mathcal{V}$ is
\begin{equation*}
\PPi_{\xi, \lambda, \mathcal{V}} \P := \argmin_{\Q \in \mathcal{V}} \llxi(\P, \Q) \; ,
\end{equation*}
where both the projection and the distance are in bold to distinguish them from projection and distances in distribution space.
\end{defn}

In particular, two projections are of interest. First, we consider the set of all value distributions from $\mathcal{X}$ to distributions supported by $\z$. The projection onto this set is
\begin{align*}
[\PPi_{\xi, \lambda, \cD} \P](x) &= \Pi_{\lambda, \cD} \P(x) \; ,
\end{align*}
We are also interested in the $\xi$-weighted projection onto $\Phi$, the set of linear value distributions:
\begin{equation*}
\PPi_{\xi, \lambda, \Phi} \P := \argmin_{\Phi \Theta, \Theta \in \bR^{m\times k}} \lxi(\P, \Theta \Phi) ,
\end{equation*}

The projection $\PPi_{\xi, \lambda, \Phi}$ of the true value distribution $\Q$ gives us the closest linear value distribution according to the \cramer distance defined by $C_\lambda$.

\begin{restatable}[Projection onto $\Phi$]{lemma}{projectionphi}
Let $\P$ be an arbitrary value distribution supported on $\bP$. The $\xi$-weighted projection of $\P$ onto $\Phi$, $\PPi_{\xi, \lambda, \Phi} \P$,  is equal to the $\xi$-weighted projection of $\PPi_{\xi, \lambda, \cD} \P$.
\end{restatable}
The above lemma will let us restrict our attention to distributions on $\z$, that is $\P \in \cD$.
\begin{proof}
Fix $\Q := \Phi \Theta$. By definition, the support of $\Q$ is $\cP$. Now
\begin{align*}
\llxi(\P, \Q) &= \sum_{x \in \cX} \xi(x) \ldt(\P(x), \Q(x)) \\
&= \sum_{x \in \cX} \xi(x)\ldt(\Q(x), \Pi_{\lambda,\cD} \P(x))\\
&\quad + \sum_{x \in \cX} \xi(x)\ldt(\Pi_{\lambda,\cD} \P(x), \P(x)) \\
&= \lxi(\Q, \Pi_{\lambda,\cD} \P) + \\
  &\qquad \sum_{x \in \cX} \xi(x) \ldt(\Pi_{\lambda,\cD} \P(x), \P(x)),
\end{align*}
using Eq.~\ref{eq:rowland}. From the above we deduce that the matrix $\Theta$ which minimizes
$\lxi(\Phi \Theta, \P)$ is also the minimizer of $\lxi(\Phi \Theta, \Pi_{\lambda,\cD} \P)$.
\end{proof}

\begin{restatable}[$\PPi_{\xi, \lambda, \Phi}$ is a non-expansion]{lemma}{contraction}
$\lPPi$ is a non-expansion in $\llxi$, i.e. for every pair $(\P, \Q)$ of value distributions, we have
\begin{align*}
\llxi(\PPi_{\xi, \lambda, \Phi} \P, \PPi_{\xi, \lambda, \Phi} \Q) &\le \llxi (\P, \Q) \; .
\end{align*}
\begin{proof}
We can view $\llxi$ as a weighted $L_2$ norm over vectors in $\bR^{n \times k}$, with $\lPPi$ the corresponding projection onto the affine subspace spanned by $\Phi$. The result is standard from these observations.
\end{proof}
\end{restatable}
Recall that the loss $\ldl$ between vectors is defined through the matrix $C_\lambda = \Pi_{e^\perp} CC^\top \Pi_{e^\perp} + \lambda e e^\top$: $\ldl(\p, \q) = (\p - \q)^\top C_\lambda (\p - \q) = \norm{\p - \q}^2_{C_\lambda}$. To prove Theorem \ref{theorem:convergencedistrib} we will consider two separate components of that loss: along $e$ and along the subspace orthogonal to $e$. That is, let us write
\begin{equation*}
A := \Pi_{e^\perp} C,
\end{equation*}
such that
\begin{equation*}
\ldl(\p, \q) = \norm{\p - \q}^2_{\Aat} + \lambda \norm{\p - \q}^2_{\eet}.
\end{equation*}
We extend this notation to a $\xi$-weighted norm over value distributions. For a matrix $B \in \bR^{k \times k}$ and $\Delta \in \bR^{n \times k}$ write
\begin{equation*}
\norm{\Delta}^2_{\xi, B} = \sum_{x \in \cX} \xi(x) \norm{\Delta(x)}^2_B,
\end{equation*}
where we associate each state $x \in \cX$ with an integer in $\{1, \dots, n\}$. Then:
\begin{equation*}
\llxi(\P, \Q) = \norm{\P - \Q}^2_{\xi, \Aat} + \lambda \norm{\P - \Q}^2_{\xi, \eet} .
\end{equation*}

\xiweighted*
Lemma \ref{lemma:xiweighted} states that the distributional Bellman operator, applied over distributions in $\bR^{n \times k}$, contracts all dimensions orthogonal to $e$ by a factor $\gamma^{1/2}$ but is only a nonexpansion along $e$.
\begin{proof}
Let $\P, \Q$ be two value distributions. To keep the notation light, without loss of generality let $\lambda = 1$. We begin with the term in $\eet$:
\begin{align*}
&\norm{\cTpi' \P - \cTpi' \Q}^2_{\xi, \eet}\\
&\qquad = \sum_{x \in \cX} \xi(x) \norm{\cTpi' \P(x) - \cTpi' \Q(x)}^2_{\eet} \\
&\qquad = \sum_{x \in \cX} \xi(x) \norm{e^\top \cTpi' \P(x) - e^\top \cTpi' \Q(x)}^2 .
\end{align*}
The term $e^\top \cTpi' \P(x)$ measures the total mass at $x$ (up to a multiplicative constant $\sqrt{k}$), after applying the distributional Bellman operator $\cTpi$ and projecting onto the finite support. $\cTpi \P(x)$ consists of a mixture of next-state distributions, shifted by the reward $r(x)$ and scaled by the discount factor $\gamma$. However, neither of these two operations affects the mass of the distributions. Furthermore, the Cram\'er projection onto the support also preserves mass \citep{rowland2018}. Hence
\begin{equation*}
e^\top \cTpi' \P(x) = \sum_{x' \in \cX} \Ppi(x' \cbar x) e^\top \P(x') .
\end{equation*}
And therefore
\begin{align*}
&\sum_{x \in \cX} \xi(x) \norm{\cTpi' \P(x) - \cTpi' \Q(x)}^2_{\eet}\\
&= \sum_{x \in \cX} \xi(x) \left(\sum_{x' \in \cX} \Ppi(x' \cbar x) e^\top \P(x') - \right. \\
  & \qquad\qquad\qquad\left. \Ppi(x' \cbar x) e^\top \Q(x')\right)^2 \\
&= \sum_{x \in \cX} \xi(x) \left(\sum_{x' \in \cX} \Ppi(x' \cbar x) e^\top (\P(x') - \Q(x'))\right)^2 .
\end{align*}
Now, by Jensen's inequality and the fact that $\xi(x') = \sum\nolimits_{x} \xi(x) \Ppi(x' \cbar x)$,
\begin{align*}
&\sum_{x \in \cX} \xi(x) \big(\sum_{x' \in \cX} \Ppi(x' \cbar x) e^\top (\P(x') - \Q(x'))\big)^2\\
&\qquad \le \sum_{x \in \cX} \xi(x) \sum_{x' \in \cX} \Ppi(x' \cbar x) (e^\top (\P(x') - \Q(x')))^2 \\
&\qquad = \sum_{x' \in \cX} \xi(x') \Ppi(x' \cbar x) (e^\top (\P(x') - \Q(x')))^2 \\
&\qquad = \norm{\P - \Q}^2_{\xi, \eet} .
\end{align*}
This proves the second statement. For the first, notice that we can add any constant vector $\alpha(x) e$ to the distribution at each state, without changing the $\Aat$-distance between them:
\begin{align*}
\norm{\cTpi' \P - \cTpi' \Q}^2_{\xi, \Aat} &= \norm{\cTpi' (\P + \alpha e) - \cTpi' \Q}^2_{\xi, \Aat} .
\end{align*}
In particular, we can choose $\alpha e$ so that the two value distributions have equal mass at all states (and in fact, sum to 1 at all states, by also changing $\Q$). In turn we can modify results by \citet{bellemare2017cramer} and \citet{rowland2018} showing that the distributional Bellman operator, projected onto a finite support or not, is a $\gamma^{1/2}$ contraction in Cram\'er metric, extending it as above to deal with the $\xi$-weighted norm rather than the maximal norm. We conclude that
\begin{equation*}
\norm{\cTpi' \P - \cTpi' \Q}^2_{\xi, \Aat} \le \gamma \norm{\P - \Q}^2_{\xi, \Aat} . \qedhere
\end{equation*}
\end{proof}

\convergencedistrib*
\begin{proof}[Proof (Sketch)]
To prove the theorem, we cannot make direct use of the usual techniques e.g. from \citet{tsitsiklis97analysis}. First, the operator $\nPPi$ is not a projection operator when $\lambda > 0$, because of the normalization term $\lambda (\q^\top e - 1)^2$ (Equation \ref{eq:unit_norm}). Second, the Bellman operator is not a contraction when applied to distributions with varying mass.

Let us consider two process $\P_{k+1} = \nPPi \cTpi \P_k$ and $\Q_k = \nPPi \cTpi Q_k$, possibly with different initial conditions. We make use of the following fact:
\begin{equation*}
\nPPi \cTpi \P = \lPPi \tilde \cTpi \P,
\end{equation*}
where $\tilde \cTpi \P = \Pi_{e^\perp} \cTpi \P + \frac{e}{\sqrt{k}}$ is a modification of the distributional Bellman operator which ``resets'' the mass of the resulting distribution to 1 by adding the appropriate constant vector (recall $e = [ 1/\sqrt{k}, \dots, 1/\sqrt{k} ]^\top$). We use this fact to measure how the two processes evolve under the norm $\norm{\cdot}_{\xi, C_\lambda}$:

\begin{align*}
\norm{\P_{k+1} - \Q_{k+1}}^2_{\xi, C_\lambda} &= \\
  &\hspace{-8em} = \norm{\nPPi \cTpi \P_k - \nPPi \cTpi \Q_{k}}^2_{\xi, C_\lambda} \\
&\hspace{-8em} = \norm{\lPPi \tilde \cTpi \P_k - \lPPi \tilde \cTpi \Q_k}^2_{\xi, C_\lambda} \\
&\hspace{-8em} \le \norm{\tilde \cTpi \P_k - \tilde \cTpi \Q_k}^2_{\xi, C_\lambda} \\
&\hspace{-8em} = \norm{\tilde \cTpi \P_k - \tilde \cTpi \Q_k}^2_{\xi, \Aat} + \norm{\tilde \cTpi \P_k - \tilde \cTpi \Q_k}^2_{\xi, \eet} \\
&\hspace{-8em} = \norm{\cTpi \P_k - \cTpi \Q_k}^2_{\xi, \Aat} + \norm{\cTpi \P_k - \cTpi \Q_k}^2_{\xi, \eet},
\end{align*}
where the last line follows from the fact that the addition of the constant $e / \sqrt{k}$ does not impact either term. Furthermore, 
\begin{align*}
\norm{\tilde \cTpi \P_k - \tilde \cTpi \Q_k}^2_{\xi, \eet} &= \norm{\Pi_{e^\perp} \P_k - \Pi_{e^\perp} \Q_k}^2_{\xi, \eet} \\
&= 0 .
\end{align*}
It follows from Lemma \ref{lemma:xiweighted} that
\begin{align*}
\norm{\P_{k+1} - \Q_{k+1}}^2_{\xi, C_\lambda} &\le \norm{\tilde \cTpi \P_k - \tilde \cTpi \Q_k}^2_{\xi, \Aat} \\
&\le \gamma \norm{\P_k - \Q_k}^2_{\xi, \Aat} \\
&\le \gamma \norm{\P_k - \Q_k}^2_{\xi, C_\lambda} .
\end{align*}
Now if $\lambda > 0$, the norm $\norm{\cdot}_{\xi, C_\lambda}$ is a true norm and
\begin{equation*}
\norm{\P_k - \Q_k}^2_{\xi, C_\lambda} \to 0 \implies \P_k, \Q_k \to \tilde \P .
\end{equation*}
When $\lambda = 0$ we have no guarantees on what happens to the $e$ component of either $\P_k$ or $\Q_k$, and we can only say that $\P_k$ (resp., $\Q_k$) converges to a set $S$ whose elements differ by a constant component.

Using a variation on a standard argument \citep{tsitsiklis97analysis}, we now write (in $\xi$-weighted norm)
\begin{align*}
\llxi(\tilde \P, \P^\pi) &= \llxi(\nPPi \cTpi \tilde \P, \P^\pi) \tag{By definition of $\tilde \P$}\\
&= \llxi(\lPPi \tilde \cTpi \tilde \P, \P^\pi) \\
&= \llxi(\lPPi \tilde \cTpi \tilde \P, \lPPi \P^\pi)\\
&\qquad + \llxi(\lPPi \P^\pi, \P^\pi) \tag{Using Eq.~\eqref{eq:rowland}}\\
&= \llxi(\lPPi \tilde \cTpi \tilde \P, \lPPi \cTpi \P^\pi)\\
&\qquad + \llxi(\lPPi \P^\pi, \P^\pi) \tag{$\P^\pi$ is the fixed point of $\cTpi$} \\
&\le \llxi(\tilde \cTpi \tilde \P, \cTpi \P^\pi) + \llxi(\lPPi \P^\pi, \P^\pi) .
\end{align*}
We now focus on the first term. Unlike \citet{tsitsiklis97analysis}'s argument, we are faced here with two different operators: $\tilde \cTpi$ and $\cTpi$. We write
\begin{align*}
\llxi(\tilde \cTpi \tilde \P, \cTpi \P^\pi) &= \norm{\tilde \cTpi \tilde \P - \cTpi \P^\pi}^2_{\xi, C_\lambda} \\
&= \norm{\tilde \cTpi \tilde \P - \cTpi \P^\pi}^2_{\xi, \Aat}\\
&\qquad + \lambda \norm{\tilde \cTpi \tilde \P - \cTpi \P^\pi}^2_{\xi, \eet} .
\end{align*}
Because $\tilde \cTpi$ ``resets'' the distribution's mass to 1, the second term is zero. Similarly,
\begin{align*}
\norm{\tilde \cTpi \tilde \P - \cTpi \P^\pi}^2_{\xi, \Aat} &= \norm{\cTpi \tilde \P - \cTpi \P^\pi}^2_{\xi, \Aat} \\
&\le \gamma \norm{\tilde \P - \P^\pi}^2_{\xi, \Aat} \\
&= \gamma \norm{\tilde \P - \P^\pi}^2_{\xi, C_\lambda}\\
&\qquad - \gamma \lambda \norm{\tilde \P - \P^\pi}^2_{\xi, \eet} .
\end{align*}
Expanding the first inequality repeatedly, we put everything together and find that
\begin{align*}
\llxi(\tilde \P, \P^\pi) &\le \frac{1}{1-\gamma}\norm{\lPPi \P^\pi - \P^\pi}^2_{\xi, C_\lambda}\\
&\qquad - \frac{\gamma \lambda}{1 - \gamma} \norm{\tilde \P - \P^\pi}^2_{\xi, \eet} . \qedhere
\end{align*}
\end{proof}

\begin{cor}
Under the same conditions as those used by~\citet{tsitsiklis97analysis}, the stochastic update process where one samples $x \sim \xi$ and updates the parameter $\Theta$ according to
\begin{equation*}
\Theta_{k+1} \gets \Theta_k + \alpha_k \grad_\Theta \ldt(\hat \cTpi \P_k(x), \P_k(x)) \; ,
\end{equation*}
where $\hat \cTpi$ is the random operator derived from a sample transition $(x, r, x')$, also converges.
\end{cor}

To prove Theorem ~\ref{theorem:convergenceexpectation}, we will need the following result:
\begin{restatable}[Ratio of operators]{lemma}{ratiooperators}
\label{lemma:ratiooperators}
Let $M$ be a self-adjoint linear operator and $N$ be a self-adjoint, invertible linear operator. Then
\begin{align*}
\sup_f \frac{<f, Mf>}{<f, Nf>} &= \rho\left(N^{-1/2}MN^{-1/2}\right) \;,
\end{align*}
where $\rho(\cdot)$ denotes the spectral radius of its argument.
\end{restatable}
\begin{proof}
Denoting $g = N^{1/2}f$, we have
\begin{align*}
f &= N^{-1/2}g\\
\frac{<f, Mf>}{<f, Nf>} &= \frac{<N^{-1/2}g, MN^{-1/2}g>}{<g, g>}\\
&= \frac{<g, N^{-1/2}MN^{-1/2}g>}{<g, g>} \; .
\end{align*}
Taking the supremum over $g$ gives the desired result.
\end{proof}

\convergenceexpectation*
\begin{proof}
\begin{align*}
  &\norm{\expects_{\tilde \P} \z - V^\pi}_{\xi}^2 = \\
&\qquad= \sum_{x \in \cX}\xi(x) \langle \tilde \P(x) - \P^\pi(x), \z\z^* \big(\tilde \P(x) - \P^\pi(x)\big)\rangle\\
&\qquad= \sum_{x \in \cX}\xi(x) \langle\tilde \P(x) - \P^\pi(x), C_\lambda \big(\tilde \P(x) - \P^\pi(x)\big)\rangle\\
&\qquad\qquad \times \frac{\langle\tilde \P(x) - \P^\pi(x), \z\z^* \big(\tilde \P(x) - \P^\pi(x)\big)\rangle}{\langle\tilde \P(x) - \P^\pi(x), C_\lambda \big(\tilde \P(x) - \P^\pi(x)\big)\rangle}\\
&\qquad\le \sum_{x \in \cX}\xi(x) \langle\tilde \P(x) - \P^\pi(x), C_\lambda \big(\tilde \P(x) - \P^\pi(x)\big)\rangle\\
&\qquad\qquad \times\max_{f} \frac{\langle f, \z\z^* f\rangle}{\langle f, C_\lambda f\rangle}\\
&\qquad\overset{(a)}{=} \sum_{x \in \cX}\xi(x) \langle\tilde \P(x) - \P^\pi(x), C_\lambda \big(\tilde \P(x) - \P^\pi(x)\big)\rangle\\
&\qquad\qquad \times \|C_\lambda^{-1/2}\z\|^2\\
&\qquad= \|C_\lambda^{-1/2}\z\|^2 \llxi(\lPPi^\pi, \P^\pi) \;
\end{align*}
where the step a) uses Lemma~\ref{lemma:ratiooperators} and the fact that $A^{-1/2}\z\z^* A^{-1/2}$ is a rank one operator.
\end{proof}

\section{Experimental Details}\label{sec:experimental_details}

Our S51 implementation is based on the C51 code from the Dopamine framework \citet{castro18dopamine}, with only minor modifications to account for the new loss. Specifically, we
\begin{enumerate}
	\item Remove the softmax transfer function mapping logits to probabilities; our network's outputs $o(x, a)$ are directly used as ``probabilities'';
	\item Select actions according to the maximum predicted ``expectation'', which is $\z^\top o(x,a)$, where $\z$ is a 51-dimensional vector whose entries are uniformly spaced within $[-10, 10]$;
	\item Replace the cross-entropy loss by the modified squared loss defined in Equation \ref{eq:unit_norm}.
\end{enumerate}
For C51, we used the hyperparameters provided by \citet{bellemare2017distributional}. We optimized the hyperparameters for S51 over the same range as used in that paper, and found that a smaller step size ($\alpha = 2.5 \times 10^{-5}$, vs. $2.5 \times 10^{-4}$ for C51) and optimizer epsilon ($\epsilon_{\textsc{opt}} = 3.125 \times 10^{-5}$, vs $3.125 \times 10^{-4}$) performed best. The parameter $\lambda = 10$ was selected from a hyperparameter sweep ($\lambda \in \{0, 0.25, 1, 10, 20, 100\}$); we found the method to perform reasonably the same for a broad range of $\lambda$ values, but note that $\lambda = 0$ yielded worse performance. In both cases, the training epsilon was set to $\epsilon = 0.05$, and lives lost were counted as the end of an episode.

\begin{figure}[htb]
\begin{tabular}{ l | r}
	\textbf{\textsc{games}}	&	\textbf{\textsc{video}} \textbf{\textsc{url}}\\
	\hline
	Asterix	& \url{https://youtu.be/hk4sYkx-VuQ} \\
	Breakout		& \url{https://youtu.be/POWvu9-2m6E} \\
	Pong	& \url{https://youtu.be/f63K_peZ6uE} \\
	Seaquest & \url{https://youtu.be/lbySDvtAmPo} \\
	Space Invaders & \url{https://youtu.be/dMvN9gmAy7E}
\end{tabular}
\caption{Links to videos of the S51 value distributions after training.\label{fig:videos}}
\end{figure}

\end{document}

%% file: notation.tex
\newcommand{\cX}{\mathcal{X}}

\newcommand{\cA}{\mathcal{A}}

\newcommand{\cP}{\mathcal{P}}

\newcommand{\cT}{\mathcal{T}}

\newcommand{\bR}{\mathbb{R}}

\DeclareMathOperator*{\expect}{{\huge \mathbb{E}}}
\newcommand{\expects}{\expect\nolimits}

\newcommand{\norm}[1]{\left \| #1 \right \|}

\newtheorem{defn}{Definition}

\newtheorem{cor}{Corollary}

\newcommand{\cbar}{\, | \,}
\newcommand{\cdbar}{\, \| \,}

\def \grad {\nabla}

\DeclareMathOperator*{\argmin}{arg\,min}

\newcommand{\eqnref}[1]{(\ref{eqn:#1})}